\documentclass{ws-ijprai}

\usepackage{natbib}
\usepackage{subfigure}

\usepackage{tikz}
\usetikzlibrary{matrix,calc,shapes}
\tikzset{
  treenode/.style = {shape=rectangle, rounded corners,
                     draw, anchor=center,
                     text width=17em, align=center,
                     inner sep=2ex},
  decision/.style = {treenode, diamond, aspect=2, inner sep=-10pt},
  root/.style     = {treenode,text width=15em},
  env/.style      = {treenode,text width=15em},
  finish/.style   = {treenode,text width=15em}
}

\newcommand{\no}{edge  node [left]  {no}}

\newcommand{\ind}{\mathrel{\perp\mspace{-10mu}\perp}}

\newtheorem{note}{Note}

\begin{document}

%\linenumbers
%\large

\markboth{N.~Gribkova and R.~Zitikis}{Functional Correlations and Classifiers}

%%%%%%%%%%%%%%%%%%%%% Publisher's area please ignore %%%%%%%%%%%%%%%
%
\catchline{}{}{}{}{}
%
%%%%%%%%%%%%%%%%%%%%%%%%%%%%%%%%%%%%%%%%%%%%%%%%%%%%%%%%%%%%%%%%%%%%

\title{Functional Correlations in the Pursuit of Performance Assessment of Classifiers}

\author{Nadezhda Gribkova}

\address{Faculty of Mathematics and Mechanics, Saint Petersburg State University,\\ Saint Petersburg 199034, Russia\\
n.gribkova@spbu.ru}

\author{Ri\v cardas Zitikis}

\address{School of Mathematical and Statistical Sciences,
Western University,\\ London, Ontario N6A 5B7, Canada\\
rzitikis@uwo.ca}

\maketitle

%begin{history}
%received{(Day Month Year)}
%revised{(Day Month Year)}
%accepted{(Day Month Year)}
%comby{(xxxxxxxxxx)}
%end{history}

\begin{abstract}
In statistical classification and machine learning, as well as in social and other sciences,  a number of measures of association have been proposed for assessing and comparing individual classifiers, raters, as well as their groups. In this paper, we introduce, justify, and explore several new measures of association, which we call CO-, ANTI- and COANTI-correlation coefficients, that we demonstrate to be powerful tools for classifying confusion matrices. We illustrate the performance of these new coefficients using a number of examples, from which we also conclude that the coefficients are new objects in the sense that they differ from those already in the literature.
\end{abstract}


\section{Introduction}
\label{sect-1}

Classification has been a massively important and rapidly developing area of research, and for a glimpse of recent developments on the topic, we refer, for example, to selected contributions presented at the two latest scientific meetings of the Classification and Data Analysis Group (CLADAG) of the Italian Statistical Society \citep[e.g.,][and references therein]{MCV18, GDBV19}.

In the present paper, we focus on ordinal classification, which is a fundamental problem in supervised learning.  Many classification algorithms,
including support vector machines \citep{V00},
random forests \citep{B01},
and instance-based methods \citep[e.g.,][and references therein]{GSS19}
have been modified to handle problems with ordered classes \citep{CK07, HS04, JTB16}.

Model assessment is an important task in supervised classification. In binary or multi-class problems, a common approach to model evaluation is to use confusion matrices and calculate accuracy measures such as error rate, precision and recall. Other approaches include graphical assessment methods such as the ROC curve and reliability diagram \citep{GY19}. When classification problems contain ordered classes, employing those metrics faces difficulties in evaluating classifiers because they ignore the ordering information in the response variable.

For an effective and consistent evaluation for ordinal classification methods, \citet{CP07} and \citet{CS11} have  employed concordant and discordant vectors. These studies have inspired our current work that explores the potential of using various facets of the notion of comonotonicity
\citep[e.g.,][and references therein]{WZ20} for tackling classification problems.

Namely, let $\{1,\dots ,d\}$ be the set of ordinal classes to which subjects or objects $\omega \in \Omega$ are classified. For example, we may think of six ($d=6$) education levels: high school, associate's, bachelor's, master's, professional, and doctoral \citep[e.g.,][]{T18}. Assume that there are two classifiers, $X$ and $Y$, which are mappings from $\Omega$ to $\{1,\dots ,d\}$ and can therefore be formally viewed as random variables defined on the sample space $\Omega$. To exemplify, we may think of $X$ and $Y$ as the highest education levels achieved by parents and their offsprings, respectively.

With $p_{i,j}$ denoting the joint probability $\mathbb{P}(X=i,Y=j)$, we obtain the matrix
\begin{equation}\label{matrix-0}
M_{X,Y}= (p_{i,j})_{i,j=1}^d ,
\end{equation}
which is frequently called the confusion matrix in statistical classification. The marginal distributions of $X$ and $Y$ are given by the probabilities $p_{i\bullet}=\sum_{j=1}^d p_{ij}$ and $p_{\bullet j}=\sum_{i=1}^d p_{ij}$, respectively. For illustrative examples, we refer to studies of social mobility \citep[e.g.,][]{TS77,CJL09,BH06}, and problems of psychometry \citep[e.g.,][]{C68,W13b}. Later in this paper, we shall revisit confusion matrices analyzed by \citet{CS11}.

Associated with the classifier $X$, there is a valuation function $f:\{1,\dots ,d\}\to \mathbb{R}$. For example, we may think of the parent's salary $f(X)$ corresponding to the education level $X$. Likewise, associated with the classifier $Y$, there is a (possibly different) valuation function $g:\{1,\dots ,d\}\to \mathbb{R}$, and we may think of $g(Y)$ as the offspring's salary corresponding to the education level $Y$. From these valuations, we obtain the $d$-variate valuation vectors $\mathbf{f}=(f_i)^{\top}$ and $\mathbf{g}=(g_i)^{\top}$, where $f_i=f(i) $ and $ g_j=g(j)$. The correlation coefficient between the two valuations with respect to the confusion matrix $M_{X,Y}$ is
\begin{equation}\label{corr-fg}
C(\mathbf{f},\mathbf{g})
={\sum_{i,j=1}^d f_i p_{ij} g_j -\big (\sum_{i=1}^d f_i p_{i\bullet}\big) \big ( \sum_{j=1}^d g_j p_{\bullet j} \big )\over
\sqrt{\sum_{i=1}^d f_i^2 p_{i\bullet}-\big ( \sum_{i=1}^d f_i p_{i\bullet}\big )^2 }
~ \sqrt{\sum_{j=1}^d g_j^2 p_{\bullet j}-\big ( \sum_{j=1}^d g_j p_{\bullet j }\big )^2 } } \in [-1,1],
\end{equation}
which is well defined whenever $\mathbf{f},\mathbf{g}\in \mathbb{R}^d$ are
non-degenerate, that is, when the condition $\sum_{i=1}^d f_i^2 p_{i\bullet}-( \sum_{i=1}^d f_i p_{i\bullet} )^2>0$ holds for $\mathbf{f}$, and an analogous condition holds for $\mathbf{g}$.

The coefficient $C(\mathbf{f},\mathbf{g})$ assesses correlation at the individual level, that is, for each pair $\mathbf{f}$ and $\mathbf{g}$ of valuations, such as, using our earlier illustrative interpretation,  the salaries corresponding to the highest education levels $X$ and $Y$ achieved by parents and their offsprings. Note, however, that salaries (usually) vary within certain ranges, depending on the education level. Keeping this in mind, instead of dealing with the correlation $C(\mathbf{f},\mathbf{g})$ for any specific pair of valuations $\mathbf{f}$ and $\mathbf{g}$, we deal with sets of valuations (for example, salary brackets) that satisfy certain order requirements.

For example, assuming that higher levels of education result in higher levels of pay, we naturally deal with ``increasing-increasing'' (II-, for short) valuations, that is, with those $\mathbf{f}$ and $\mathbf{g}$ for which the inequality constraints $f_i\le f_j$ and $g_i\le g_j$ hold whenever $1\le i \le j \le d$. The maximal correlation coefficient over this set of valuations $\mathbf{f}$ and $\mathbf{g}$ gives rise to what is known in the literature as the II-correlation coefficient \citep{KS78,KMS82}. Likewise, we have the ``increasing-decreasing'' (ID-) correlation coefficient. Combining these two coefficients, we obtain the monotone (MON-) correlation coefficient. We refer to \citet{KS78}, and \citet{KMS82} for properties and other details related to these correlation coefficients, but we shall also discuss them later in this paper. Note that when we maximize $C(\mathbf{f},\mathbf{g})$ with respect to all non-degenerate pairs $\mathbf{f}$ and $\mathbf{g}$, we arrive at the supremum (SUP-) correlation coefficient of \citet{G41}.

Following the terminology of \citet{S86} and \citet{D94}, we call the vectors $\mathbf{f}$ and $\mathbf{g}$ comonotone when $(f_i-f_j)(g_i-g_j)\ge 0$ for all  $1\le i,j \le d$, and antimonotone when $(f_i-f_j)(g_i-g_j)\le 0$  for all  $1\le i,j \le d$. Maximizing $C(\mathbf{f},\mathbf{g})$ with respect to all non-degenerate comonotone (resp.~antimonotone) pairs $(\mathbf{f},\mathbf{g})$ gives rise to what we call comonotone (CO-) and antimonotone (ANTI-) correlation coefficients, respectively. We shall show in the following sections that the latter two coefficients as well as a third one that we call the COANTI-correlation coefficient (to be introduced in Section~\ref{sect-3} below) are new notions that, in general, differ from any of the known functional correlation coefficients, such as the aforementioned II-, ID-, and MON-correlation coefficients.

To start appreciating these new coefficients and see their potential in classification, we note that comonotone pairs $(\mathbf{f},\mathbf{g})$ are common and natural valuations in the context of social mobility. For example, referring to the aforementioned example concerning the education levels and salaries, we see from the charts provided by \citet{T18} that salaries do not always increase when educational levels increase. Specifically, a bar chart reported by \citet{T18} shows that among the aforementioned six levels of education, the median weekly earnings peek at the penultimate (i.e., professional) degree, thus making the earlier used II-correlation coefficient suboptimal, while the CO-correlation coefficient perfectly suits the purpose.

The rest of the paper is organized as follows.
In Section~\ref{sect-2}, we discuss the classical II-, ID- and MON-correlation coefficients in the case of generic classifiers $X$ and $Y$.
In Section~\ref{sect-3}, we formally define and discuss the already mentioned CO-, ANTI- and COANTI-correlation coefficients,  which are new measures of association. 
In Section~\ref{sect-4}, we provide a computational method for all (new and old) functional correlation coefficients.
In Section~\ref{sect-5}, we put forward a flowchart for comparing confusion matrices and then use it to revisit matrices of \citet{CS11}. We shall see from the numerical results obtained therein that the newly introduced functional correlation coefficients, in conjunction with the classical ones, provide a powerful and unifying approach to classifying confusion matrices. 
Section~\ref{sect-conclude} concludes the paper with a brief summary of main contributions and some remarks.

\section{Classical functional correlation coefficients}
\label{sect-2}

The Pearson correlation coefficient
\[
\varrho(X,Y)=\mathrm{Corr}[X,Y] \in [-1,1]
\]
has played a pivotal role in various research areas. It is not, however, an independence determining coefficient, that is, the equation $\varrho(X,Y)=0$ does not, in general, imply that $X$ and $Y$ are independent, henceforth shorthanded as $ X \ind Y$. To achieve independence determination, the functional correlation coefficient \citep[cf.][]{S92}
\begin{equation}\label{general}
\varrho(X,Y \mid \mathcal{A})=\sup_{(g,h)\in \mathcal{A}}\mathrm{Corr}[g(X),h(Y)] \in [-1,1]
\end{equation}
becomes a natural tool, defined for various subsets $\mathcal{A}$ of the set
\[
\mathcal{B}
=\big \{ (g,h) ~:~ \mathrm{Var}[g(X)],\mathrm{Var}[h(Y)]\in (0,\infty) \big \}.
\]
Throughout the paper, we deal only with Borel functions, and thus $g$ and $h$ in the definition of $\mathcal{B}$ are tacitly assumed to be such. Note that the set $\mathcal{B}$ depends on the cumulative distribution functions (cdf's) $F_X$ and $F_Y$ of $X$ and $Y$, respectively. Obviously, when $\mathcal{A}=\mathcal{B}$, the coefficient $\varrho(X,Y \mid \mathcal{A})$ is the largest and  therefore called the supremum (SUP-) correlation coefficient \citep{G41}, defined by the equation
\[
\varrho_{\textsc{sup}}(X,Y)=\varrho(X,Y \mid \mathcal{B}) \in [0,1].
\]

\begin{note}
Unless $\mathrm{Var}[X]\in (0,\infty) $ and $\mathrm{Var}[Y]\in (0,\infty) $, the Pearson correlation coefficient $\varrho(X,Y)$ does not exist, but the functional correlation coefficient $\varrho(X,Y \mid \mathcal{B})$ may nevertheless be well defined. For example, when $X=Y\sim \text{Cauchy}(0,1)$, the classical Pearson coefficient $\varrho(X,Y)$ does not exist,  but the functional coefficient $\varrho(X,Y \mid \mathcal{B})$ is finite and equal to $1$.
\end{note}

The SUP-correlation coefficient $\varrho_{\textsc{sup}}(X,Y)$ is independence determining because $\mathcal{B} \supseteq \mathcal{B}_1 $, where
\[
\mathcal{B}_1
=\big \{ (\mathbb{I}_{(-\infty, x]}, \mathbb{I}_{(-\infty, y]}\big )
~:~ x,y\in \mathbb{R},~  F_X(x), F_Y(y)\in (0,1) \big \}
\]
with $\mathbb{I}_{(-\infty, x]}$ denoting the indicator function of the interval $(-\infty, x]$, that is, $\mathbb{I}_{(-\infty, x]}(z)$ is equal to $1$ when $z\in (-\infty, x]$ and $0$ when $z \notin (-\infty, x]$. Indeed, for a set $\mathcal{A}$ to be independence determining, the necessary and sufficient condition is $\mathcal{A} \supseteq \mathcal{B}_1 $, under which we have the equivalence relationship
\begin{equation}\label{equiv}
\varrho(X,Y \mid \mathcal{A})=0 \quad \Longleftrightarrow \quad X \ind Y.
\end{equation}

When the set $\mathcal{A}$ is the singleton
\begin{equation}\label{class-b0}
\mathcal{B}_0
=\big \{ (g_0,h_0) ~:~ \mathrm{Var}[X],\mathrm{Var}[Y]\in (0,\infty) \big \}
\end{equation}
with $g_0(x)=x$ and $ h_0(x)=x$ for all $x\in \mathbb{R}$, coefficient~(\ref{general}) reduces to the Pearson correlation coefficient, that is, $\varrho(X,Y)=\varrho(X,Y \mid \mathcal{B}_0)$.
The singleton $\mathcal{B}_0$ is not, however, independence determining, which is of course a well known fact.

When $\mathcal{A}$ is $\mathcal{B}_{\textsc{mon}}$ consisting of all the pairs $(g,h)\in \mathcal{B}$ of monotone functions $g$ and $h$, we have the monotone (MON-) correlation coefficient  \citep{KS78}
\[
\varrho_{\textsc{mon}}(X,Y)=\varrho(X,Y \mid \mathcal{B}_{\textsc{mon}}) \in [0,1].
\]
It is an independence determining coefficient because $ \mathcal{B}_{\textsc{mon}}\supset \mathcal{B}_1 $.

Since $\mathcal{B}_{\textsc{mon}}\subset \mathcal{B}$, we have
$\varrho_{\textsc{mon}}(X,Y)\le \varrho_{\textsc{sup}}(X,Y)$. To compare $\varrho_{\textsc{mon}}(X,Y)$ with the Pearson correlation coefficient $\varrho(X,Y)$, we need to assume $\mathrm{Var}[X]\in (0,\infty) $ and $\mathrm{Var}[Y]\in (0,\infty) $. In this case, we have the inclusions $\mathcal{B}_0\subset \mathcal{B}_{\textsc{mon}}\subset \mathcal{B}$ and thus the ordering
$\varrho(X,Y) \le \varrho_{\textsc{mon}}(X,Y)\le \varrho_{\textsc{sup}}(X,Y)$. In fact, the following inequalities hold:
\begin{equation}\label{corr-bounds}
\big |\varrho(X,Y) \big |
\le \varrho_{\textsc{mon}}(X,Y)
\le \varrho_{\textsc{sup}}(X,Y) .
\end{equation}
We can easily find random variables $X$ and $Y$ such that $|\varrho(X,Y) |< \varrho_{\textsc{mon}}(X,Y)$. We also know from \citet{KS78} that there are $X$ and $Y$
such that $\varrho_{\textsc{mon}}(X,Y) < \varrho_{\textsc{sup}}(X,Y)$. Hence, in general, the three correlation coefficients in (\ref{corr-bounds}) are distinct.

The set $\mathcal{B}_{\textsc{mon}}$ is the union of i) $\mathcal{B}_{\textsc{ii}}$ consisting of  the pairs $(g,h)\in \mathcal{B}$ of increasing functions $f$ and $g$, ii) $\mathcal{B}_{\textsc{id}}$ consisting of  the pairs $(g,h)\in \mathcal{B}$ of increasing functions $f$ and decreasing functions $g$, and iii) the sets $\mathcal{B}_{\textsc{di}}$ and $\mathcal{B}_{\textsc{dd}}$ defined analogously. Since $\mathrm{Corr}[g(X),h(Y)]$ is equal to $\mathrm{Corr}[-g(X),-h(Y)]$, we have the equations
$\varrho(X,Y \mid \mathcal{B}_{\textsc{di}})=\varrho(X,Y \mid \mathcal{B}_{\textsc{id}})$ and
$\varrho(X,Y \mid \mathcal{B}_{\textsc{dd}})=\varrho(X,Y \mid \mathcal{B}_{\textsc{ii}})$, which effectively leaves us with only two correlation coefficients \citep{KMS82}:
\begin{align*}
\varrho_{\textsc{ii}}(X,Y)   &=\varrho(X,Y \mid \mathcal{B}_{\textsc{ii}}) \in [-1,1],
\\
\varrho_{\textsc{id}}(X,Y) &=\varrho(X,Y \mid \mathcal{B}_{\textsc{id}}) \in [-1,1].
\end{align*}
In the terminology of \citet{KMS82}, $\varrho_{\textsc{ii}}(X,Y)$ is the concordant monotone correlation coefficient, and $-\varrho_{\textsc{id}}(X,Y)$ is the discordant monotone correlation coefficient. Throughout the paper, we call $\varrho_{\textsc{ii}}(X,Y)$ the II-correlation coefficient and $\varrho_{\textsc{id}}(X,Y)$ the ID-correlation coefficient. \citet{KMS82} note the equation
\begin{equation}\label{st-2a}
\varrho_{\textsc{mon}}(X,Y)=\max \big \{\varrho_{\textsc{ii}}(X,Y),\varrho_{\textsc{id}}(X,Y)  \big \} \in [0,1].
\end{equation}

The three functional correlation coefficients making up equation (\ref{st-2a}) have played a considerable role in understanding social mobility \citep[e.g.,][]{KMS82,S92}. While working on the current paper and testing various computational algorithms, we explored a number of social mobility  tables, such as those reported by \citet{TS77}, \citet{CJL09}, and \citet{BH06}. Interestingly, all of those tables have shown that their SUP-correlation coefficients are achieved on increasing pairs $g$ and $h$, thus implying that for those tables, the II-, MON- and SUP-correlation coefficients are the same or, to be precise, coincide at least up to the sixth decimal digit, which is our default computational precision throughout this paper. (Obtaining closed-form formulas for any of the functional correlation coefficients seems to be an impossible task, which is not really needed as far as we can see.) Finally, we note that during our numerical explorations, we found the computational ideas of \citet{KMS81} particularly helpful.

\section{New comonotonicity-related functional correlation coefficients}
\label{sect-3}

As we have noted above, the functional correlation coefficients corresponding to the confusion matrices arising from some intergenerational mobility surveys are maximized on increasing $g$ and $h$. This is not, however, always the case, as we see from the data discussed by \citet{T18}, where comonotone $g$ and $h$ manifest naturally. As a reaction to this fact, we next introduce what we call CO-, ANTI- and COANTI-correlation coefficients, and in Section~\ref{sect-5} we shall demonstrate their power when comparing confusion matrices. Since analytical formulas are not available due to the nature of these coefficients, we shall offer two effective computational algorithms in Sections~\ref{sect-4} and \ref{sect-5}.

The aforementioned three correlation coefficients rely on comonotone functions, which play a prominent role in areas such as quantitative finance, insurance,
and economics \citep[e.g.,][and references therein]{FS16}.
Namely, functions $g$ and $h$ are said to be comonotone \citep{S86,D94} if
\begin{equation}\label{com-def}
\big (g(x)-g(x')\big ) \big (h(x)-h(x')\big )\ge 0
\end{equation}
holds for all $x$ and $x'$ in the joint domain of definition of both $g$ and $h$.

To the best of our knowledge, the notion of comonotonicity has not yet been fully utilized in classification problems, with initial impressive hints on the notion's potential in such problems provided by \citet{CP07} and \citet{CS11}.

\begin{definition}
Let $\mathcal{B}_{\textsc{co}}$ denote the set of all pairs $(g,h)\in \mathcal{B}$ of comonotone functions. We call
\[
\varrho_{\textsc{co}}(X,Y)=\varrho(X,Y \mid \mathcal{B}_{\textsc{co}}) \in [-1,1]
\]
the comonotone (CO-) correlation coefficient.
\end{definition}

If non-negativity in condition~(\ref{com-def}) is replaced by non-positivity, then $g$ and $h$ are said to be antimonotone. That is, $g$ and $h$ are antimonotone functions if
\begin{equation}\label{anti-def}
\big (g(x)-g(x')\big ) \big (h(x)-h(x')\big )\le 0
\end{equation}
holds for all $x$ and $x'$ in the joint domain of definition of both $g$ and $h$.

\begin{definition}
Let $\mathcal{B}_{\textsc{anti}}$ denote the set of all pairs $(g,h)\in \mathcal{B}$ of antimonotone functions. We call
\[
\varrho_{\textsc{anti}}(X,Y)=\varrho(X,Y \mid \mathcal{B}_{\textsc{anti}}) \in [-1,1].
\]
the antimonotone (ANTI-) correlation coefficient.
\end{definition}

We next relate the CO- and ANTI-correlation coefficients to the classical II- and ID-correlation coefficients.

\begin{proposition}\label{prop-2}
We have the representations
\begin{gather}
\varrho_{\textsc{co}}(X,Y) = \sup_{h}\varrho_{\textsc{ii}}(h(X),h(Y)),
\label{st-4}
\\
\varrho_{\textsc{anti}}(X,Y)= \sup_{h}\varrho_{\textsc{id}}(h(X),h(Y)),
\label{st-5}
\end{gather}
where the suprema on the right-hand sides of equations~(\ref{st-4}) and~(\ref{st-5}) are taken with respect to all functions $h$ for which the suprema are well defined.
\end{proposition}

\begin{proof}
Equation~(\ref{st-4}) follows from the fact that any two functions $f$ and $g$ are comonotone if and only if there is a function $h$ and also two increasing functions $h_1$ and $h_2$ such that $f(x)=h_1(h(x))$ and $g(x)=h_2(h(x))$ \citep[Proposition 4.5(iv), 54--55]{D94}. Equation~(\ref{st-5}) follows from equation~(\ref{st-4}) by observing that $f$ and $g$ are antimonotone if and only if $f$ and $(-g)$ are comonotone.
\end{proof}

\begin{definition}
Let $\mathcal{B}_{\textsc{coanti}}$ denote the set of all pairs $(g,h)\in \mathcal{B}$ of comonotone, as well as antimonotone, functions $f$ and $g$, and we thus collectively call them ``coanti'' functions. In turn, we call
\[
\varrho_{\textsc{coanti}}(X,Y)=\varrho(X,Y \mid \mathcal{B}_{\textsc{coanti}}) \in [0,1]
\]
the COANTI-correlation coefficient.
\end{definition}

Analogously to equation~(\ref{st-2a}), we have
\begin{equation}\label{st-2}
\varrho_{\textsc{coanti}}(X,Y)=\max \big \{\varrho_{\textsc{co}}(X,Y),\varrho_{\textsc{anti}}(X,Y)  \big \},
\end{equation}
which, in view of representations~(\ref{st-4}) and~(\ref{st-5}), gives the equation
\begin{equation}\label{st-2c}
\varrho_{\textsc{coanti}}(X,Y)=\max \big \{
\sup_{h}\varrho_{\textsc{ii}}(h(X),h(Y)),
\sup_{h}\varrho_{\textsc{id}}(h(X),h(Y)) \big \}.
\end{equation}
Equation~(\ref{st-2c}) is particularly useful when comparing the COANTI- and MON-correlation coefficients by way of comparing the right-hand sides of equations~(\ref{st-2c}) and~(\ref{st-2a}).

Since $\mathcal{B}_0\subset \mathcal{B}_{\textsc{mon}}\subset \mathcal{B}_{\textsc{coanti}}\subset \mathcal{B}$, we have the inequalities
\begin{equation}\label{corr-bounds-2z}
\varrho_{\textsc{mon}}(X,Y)
\le \varrho_{\textsc{coanti}}(X,Y)
\le \varrho_{\textsc{sup}}(X,Y)
\end{equation}
and also, assuming $\mathbb{E}[X^2]\in (0,\infty) $ and $\mathbb{E}[Y^2]\in (0,\infty) $,
\begin{equation}\label{corr-bounds-3}
\big |\varrho(X,Y) \big |
\le \varrho_{\textsc{mon}}(X,Y)
\le \varrho_{\textsc{coanti}}(X,Y)
\le \varrho_{\textsc{sup}}(X,Y).
\end{equation}
Since $\mathcal{B}_1\subset \mathcal{B}_{\textsc{mon}}\subset \mathcal{B}_{\textsc{coanti}}\subset \mathcal{B}$, the MON-, COANTI-, and SUP-correlation coefficients are independence determining, that is, equivalence relationship (\ref{equiv}) holds with $\mathcal{A}$ replaced by the respective three $\mathcal{B}$'s. The following proposition connects the MON- and COANTI-correlation coefficients.

\begin{proposition}\label{prop-1}
The COANTI- and MON-correlation coefficients are related via the equation
\begin{equation}\label{st-1}
\varrho_{\textsc{coanti}}(X,Y)=\sup_{h}\varrho_{\textsc{mon}}(h(X),h(Y)),
\end{equation}
where the supremum on the right-hand side of equation~(\ref{st-1}) is taken with respect to all functions $h$ for which the supremum is well defined.
\end{proposition}

\begin{proof}
We start with equation~(\ref{st-2c}) and assume for the sake of argument that $\varrho_{\textsc{coanti}}(X,Y)$ is maximized by $\sup_{h}\varrho_{\textsc{ii}}(h(X),h(Y))$. This implies that $\sup_{h}\varrho_{\textsc{ii}}(h(X),h(Y))$ is not smaller than $\varrho_{\textsc{id}}(h_{*}(X),h_{*}(Y))$ for every function $h_{*}$ for which the latter correlation coefficient is well defined. Hence, we have the equation
\begin{equation}
\sup_{h}\varrho_{\textsc{ii}}(h(X),h(Y))
= \sup_{h}\max \big \{\varrho_{\textsc{ii}}(h(X),h(Y)),\varrho_{\textsc{id}}(h_{*}(X),h_{*}(Y)) \big \}.
\label{st-21}
\end{equation}
Since equation~(\ref{st-21}) holds for every $h_{*}$, we can choose $h_{*}$ to be $h$. We obtain the equations
\begin{align}
\sup_{h}\varrho_{\textsc{ii}}(h(X),h(Y))
&= \sup_{h}\max \big \{\varrho_{\textsc{ii}}(h(X),h(Y)),\varrho_{\textsc{id}}(h(X),h(Y)) \big \}
\notag
\\
&= \sup_{h}\varrho_{\textsc{mon}}(h(X),h(Y)),
\label{st-22}
\end{align}
where the right-most equation holds due to equation~(\ref{st-2a}). Hence, $\varrho_{\textsc{coanti}}(X,Y)$ is equal to the right-hand side of equation~(\ref{st-2a}). We also arrive at the same conclusion when $\varrho_{\textsc{coanti}}(X,Y)$ is maximized by $\sup_{h}\varrho_{\textsc{id}}(h(X),h(Y))$, thus finishing the verification of Property~\ref{prop-1}.
\end{proof}

It should be noted that we have not yet provided convincing evidence that the three functional correlation coefficients (that is, CO, ANTI, and COANTI) are truly new, that is, that they differ from the classical ones of Section~\ref{sect-2}. For this reason, in the next section we shall construct a confusion matrix (Example~\ref{example-00}) that gives rise to the following (strict) inequalities:
\begin{equation}\label{coanti-0}
\varrho_{\textsc{mon}}(X,Y)
< \varrho_{\textsc{coanti}}(X,Y)
< \varrho_{\textsc{sup}}(X,Y) .
\end{equation}
The values of the three correlation coefficients have to be calculated numerically, and we shall do so at the usual for this paper precision of six decimal digits, but we shall see (equation~\eqref{eq-00} below) differences in their values at the first decimal digit.

We conclude the current section by enhancing our intuitive appreciation of the three functional correlation coefficients in (\ref{coanti-0}). For this, we first recall the six education levels noted earlier: high school, associate's, bachelor's, master's, professional, and doctoral \citep[e.g.,][]{T18}. Suppose next that we are interested in determining the strongest possible correlation between all the possible salaries of, say, daughters and mothers depending on their education levels.

If we assume that salaries increase when the education level increases, then it is appropriate to use the (classical) coefficient $\varrho_{\textsc{ii}}(X,Y)$ or, more broadly, $\varrho_{\textsc{mon}}(X,Y)$, but as \citet{T18} illustrates, this assumption of monotonicity is sometimes too optimistic to be realistic.

Allowing for varying salary patterns within the group of daughters, as well as within the group of mothers, gives rise to the need for calculating the above introduced coefficient $\varrho_{\textsc{co}}(X,Y)$ or, more broadly, $\varrho_{\textsc{coanti}}(X,Y)$, which assesses maximal synchronicity (that is, comonotonicity) between the salaries of daughters and mothers, without subjectively imposing the opinion that the salaries are increasing when the education level increases.

Finally, the (classical) correlation coefficient $\varrho_{\textsc{sup}}(X,Y)$ allows not only for various salary patterns within the group of daughters, as well as within the group of mothers, but also various patterns between the two groups. Hence, the coefficient $\varrho_{\textsc{sup}}(X,Y)$ is not designed to assess synchronicity between the two groups -- it measures the strongest possible correlation with respect to every possible salary pattern that may arise.

We shall keep these illustrations in mind when assessing and interpreting confusions matrices in the following sections, and in particular when describing the flowchart presented in Section~\ref{sect-5.1}.

\section{Calculating functional correlation coefficients}
\label{sect-4}

In this section we work within the framework of Section~\ref{sect-1}. That is, we let $X$ and $Y$ be two classifiers taking values in the set $\{1,\dots ,d\}$. Their joint probabilities $p_{i,j}=\mathbb{P}(X=i,Y=j)$ give rise to confusion matrix~(\ref{matrix-0}).
The correlation coefficient $C(\mathbf{f},\mathbf{g})$ between the valuation vectors $\mathbf{f}=(f_i)^{\top}$ and $\mathbf{g}=(g_i)^{\top}$ is defined by formula~(\ref{corr-fg}). Without loss of generality we can, and thus do, impose the constraints
\begin{equation}\label{con-2}
\sum_{i=1}^d f_i p_{i\bullet}=\sum_{j=1}^d  g_j p_{\bullet j}=0
\quad \text{and} \quad
\sum_{i=1}^d f_i^2 p_{i\bullet}=\sum_{j=1}^d  g_j^2p_{\bullet j}=1,
\end{equation}
under which the correlation coefficient simplifies to
\begin{equation}
C(\mathbf{f},\mathbf{g})=\sum_{i,j=1}^d f_i p_{ij} g_j.
\label{cxy-eq-0}
\end{equation}

\subsection{An excursion into psychometric literature}
\label{sub-41}

Equation (\ref{cxy-eq-0}) plays a pivotal role in connecting the current research with a large body of literature dealing with the weighted kappa \citep{C68}. Indeed, keeping in mind constraints~(\ref{con-2}), equation~(\ref{cxy-eq-0}) can be rewritten as follows:
\begin{align}
C(\mathbf{f},\mathbf{g})
&=1-{1\over 2}\sum_{i,j=1}^d (f_i-g_j)^2 p_{ij}
\notag
\\
&=1-{\sum_{i,j=1}^d w_{ij} p_{ij} \over \sum_{i,j=1}^d w_{ij} p_{i\bullet}p_{\bullet j} } ,
\label{cxy-eq-000}
\end{align}
where $w_{ij}=(f_i-g_j)^2$. Usually in the literature, the vectors $\mathbf{f}$ and $\mathbf{g}$ are set to $f_i=i$ and $g_j=j$ for all $1\le i, j \le d$. Note that these special choices correspond to the class $\mathcal{B}_0$ defined by equation~(\ref{class-b0}). The weights $w_{ij}$ become quadratic $w_{ij}=(i-j)^2$ \citep{FC73}.

The right-hand side of equation~(\ref{cxy-eq-000}) defines what is known in the literature as the weighted kappa, with $w_{ij}$'s called ``disagreement'' weights, which can be any, depending on the problem at hand. For example, \citet{C60} uses  the indicator weights $w_{ij}=\mathbb{I}\{i\ne j\}$.  \citet{CA71} suggest to use the ``linear'' (also known as absolute) weights $w_{ij}=|i-j|$. Analyses and comparisons of linear and quadratic weights have been done by \citet{VA09a}, \citet{V16}, and \citet{K18}, where we also find extensive references to earlier works on the topic.

Weights $w_{ij}$ based on more general valuation vectors $\mathbf{f}$ and $\mathbf{g}$, known as category scores, have also been considered, although quite often by setting $\mathbf{f}=\mathbf{g}$. For example, \citet{VEM05}, \citet{THT12}, and \citet{K18} use disagreement weights such as $w_{ij}=|f_i-f_j|$ and $w_{ij}=(f_i-f_j)^2$ for monotone $f_1\le \cdots \le f_d$. In this paper, we do not impose requirements such as  $\mathbf{f}=\mathbf{g}$, thus allowing for the possibility of having very diverse classification patterns. Yet, as is the case with all indices, by condensing data into one number, the weighted kappa inevitably loses information, irrespective of the weights employed.

Some specific criticism has been directed toward the quadratic weights by \citet{W13a,W13b} and \citet{V16}, among others. Indeed, there are several reasons for being cautious when using such weights because they, being non-linear, distort distances between categories. Also notably, since the weighted kappa is closely related to the Pearson correlation coefficient, as pointed out by a number of researchers \citep[e.g.,][]{FC73,S04,K18}, the kappa is not an independence determining coefficient, and thus gives rise to some difficulties when measuring dependence between classifiers, as we already noted in Section~\ref{sect-1}. Fortunately, a good way out of the difficulty exists, and this brings us back to the notion of functional correlations.

\subsection{Matrix-based counterparts of the classical functional correlations}
\label{sub-42}

We start by recalling that $X\sim F_X$ and $Y\sim F_Y$ are independent if and only if the equation
\[
\mathrm{Corr}\big [\mathbb{I}_{(-\infty, x]}(X),\mathbb{I}_{(-\infty, y]}(Y)\big ]=0
\]
holds for all $x,y\in \mathbb{R} $ such that both $ F_X(x)$ and $ F_Y(y)$ are in $ (0,1)$. Consequently, when assessing association between classifiers, we can use the Pearson correlation coefficient, but we need to calculate it over a sufficiently large set of valuation vectors (category scores) $ \mathbf{f}$ and $\mathbf{g} $. This reasoning, though perhaps not always explicitly stated in the literature, has lead researchers \cite[e.g.,][]{SF73,HS78,K81,FA86} to the idea of sorting out  classification problems with the help of several indices, inncluding the Pearson correlation coefficient, Spearman's $\rho$, Kendall's $\tau$, and Spearman's footrule. A somewhat different, although closer to ours, research path has been taken by \citet{VA09b,VA09c}, and we next elaborate on it.

Within the framework of Section~\ref{sect-1}, all possible paired classifiers, also known as raters, are represented by the pairs $ (\mathbf{f},\mathbf{g}) $ of $d$-dimensional valuation vectors. Without loss of generality, we assume that constraints~(\ref{con-2}) are satisfied. We denote the set of all such pairs $ (\mathbf{f},\mathbf{g}) $ by $\mathcal{S}_{\textsc{sup}} $.  The maximal correlation coefficient over the set of all such pairs with respect to the confusion matrix $M_{X,Y}$ defines the SUP-correlation coefficient
\begin{equation}\label{cor-sup}
\varrho_{\textsc{sup}}(M_{X,Y}) = \sup_{(\mathbf{f},\mathbf{g}) \in \mathcal{S}_{\textsc{sup}}}
C(\mathbf{f},\mathbf{g}) \in [0,1].
\end{equation}

As we have noted earlier, there are good reasons for restricting the pairs  $(\mathbf{f},\mathbf{g})\in \mathcal{S}_{\textsc{sup}}$ to only those that satisfy $ f_i\le f_j $ and $  g_i\le g_j $ for all $1\le i \le j \le d$. Denote the set of all such pairs by $\mathcal{S}_{\textsc{ii}} $. The corresponding II-correlation coefficient is
\[
\varrho_{\textsc{ii}}(M_{X,Y}) = \sup_{(\mathbf{f},\mathbf{g}) \in \mathcal{S}_{\textsc{ii}}}
C(\mathbf{f},\mathbf{g}) \in [-1,1].
\]
Analogously, the ID-correlation coefficient is given by
\[
\varrho_{\textsc{id}}(M_{X,Y})= \sup_{(\mathbf{f},\mathbf{g})\in \mathcal{S}_{\textsc{id}} } C(\mathbf{f},\mathbf{g}) \in [-1,1],
\]
where $\mathcal{S}_{\textsc{id}}$ consists of all the pairs $(\mathbf{f},\mathbf{g})\in \mathcal{S}_{\textsc{sup}}$ such that $ f_i\le f_j $ and $ g_i\ge g_j $ for all $1\le i \le j \le d$. Below are two notes concerning confusion matrices that maximize the II- and ID-correlation coefficients.

\begin{note}\label{example-1a}
If the confusion matrix $M_{X,Y}$ is diagonal, that is, $p_{ij}=0$ for all $i\ne j$, then  equation~(\ref{cxy-eq-000}) implies the formula
\[
C(\mathbf{f},\mathbf{g})=1-{1\over 2}\sum_{i=1}^d (f_i-g_i)^2 p_{ii},
\]
and thus the II-correlation coefficient is equal to 1. This value is achieved on the pair $(\mathbf{f},\mathbf{g}) \in \mathcal{S}_{\textsc{ii}}$ with $f_i=g_i$ for all $1\le i \le d$.
\end{note}

\begin{note}\label{example-1b}
If the confusion matrix $M_{X,Y}$ is anti-diagonal, that is, $p_{ij}=0$ for all $i+j\ne d+1$, then equation~(\ref{cxy-eq-000}) implies the formula
\[
C(\mathbf{f},\mathbf{g})=1-{1\over 2}\sum_{i=1}^d (f_i-g_{d-i+1})^2 p_{i,d-i+1},
\]
and thus the ID-correlation coefficient is equal to 1. This value is achieved on the pair $(\mathbf{f},\mathbf{g}) \in \mathcal{S}_{\textsc{id}}$ with $f_i=g_{d-i+1}$ for all $1\le i \le d$.
\end{note}

The MON-correlation coefficient is given by the equation
\begin{equation}\label{mon-expression}
\varrho_{\textsc{mon}}(M_{X,Y}) = \sup_{(\mathbf{f},\mathbf{g}) \in \mathcal{S}_{\textsc{mon}}}
C(\mathbf{f},\mathbf{g}) \in [0,1],
\end{equation}
where $\mathcal{S}_{\textsc{mon}}$ consists of all those pairs $(\mathbf{f},\mathbf{g})\in \mathcal{S}_{\textsc{sup}}$ whose coordinates are either increasing or decreasing, that is,
$\mathcal{S}_{\textsc{mon}}$ is the union of the four sets
$\mathcal{S}_{\textsc{ii}}$,
$\mathcal{S}_{\textsc{id}}$,
$\mathcal{S}_{\textsc{di}}$, and
$\mathcal{S}_{\textsc{dd}}$.
In fact, analogous arguments to those above equation (\ref{st-2a}) show that we only need to work with the first two subsets, $\mathcal{S}_{\textsc{ii}}$ and
$\mathcal{S}_{\textsc{id}}$, thus reducing equation (\ref{mon-expression}) to
\[
\varrho_{\textsc{mon}}(M_{X,Y}) = \sup_{(\mathbf{f},\mathbf{g}) \in \mathcal{S}_{\textsc{ii}}\cup \mathcal{S}_{\textsc{id}}}
C(\mathbf{f},\mathbf{g}),
\]
which in turn gives the following equation
\begin{equation}\label{st-2aa}
\varrho_{\textsc{mon}}(M_{X,Y})
= \max \big\{ \varrho_{\textsc{ii}}(M_{X,Y}), \varrho_{\textsc{id}}(M_{X,Y})\big \}.
\end{equation}

\subsection{The new functional correlations adapted to confusion matrices}
\label{sub-43}

The three functional correlation coefficients that we introduced above are CO, ANTI, and COANTI. We shall next adapt their definitions to tackle confusion matrices. We begin with the CO-correlation coefficient and have
\[
\varrho_{\textsc{co}}(M_{X,Y}) = \sup_{(\mathbf{f},\mathbf{g}) \in \mathcal{S}_{\textsc{co}}}
C(\mathbf{f},\mathbf{g}) \in [-1,1],
\]
where $\mathcal{S}_{\textsc{co}}$ consists of all comonotone pairs $(\mathbf{f},\mathbf{g})\in \mathcal{S}_{\textsc{sup}}$, that is, of those for which the bound $(f_i-f_j)(g_i-g_j)\ge 0 $ holds for all $ 1\le i,j \le d$.

The ANTI-correlation coefficient is
\[
\varrho_{\textsc{anti}}(M_{X,Y})= \sup_{(\mathbf{f},\mathbf{g})\in \mathcal{S}_{\textsc{anti}} } C(\mathbf{f},\mathbf{g}) \in [-1,1],
\]
where $\mathcal{S}_{\textsc{anti}}$ consists of all antimonotone pairs $(\mathbf{f},\mathbf{g})\in \mathcal{S}_{\textsc{sup}}$, that is, of those for which the bound $(f_i-f_j)(g_i-g_j)\le 0  $ holds for all $ 1\le i,j \le d$.

Analogously to equation~(\ref{st-2}), the COANTI-correlation coefficient is the maximum of the CO- and ANTI-correlation coefficients, that is,
\begin{equation}\label{st-2bb}
\varrho_{\textsc{coanti}}(M_{X,Y})
= \max \big\{ \varrho_{\textsc{co}}(M_{X,Y}), \varrho_{\textsc{anti}}(M_{X,Y})\big \} \in [0,1].
\end{equation}

In view of Notes~\ref{example-1a} and~\ref{example-1b}, if the confusion matrix $M_{X,Y}$ is diagonal, then the CO-correlation coefficient is equal to 1, and if the confusion matrix is anti-diagonal, then the ANTI-correlation coefficient is equal to 1. In both of these  cases, as follows from equation~(\ref{st-2bb}), the COANTI-correlation coefficient is equal to 1.

To illustrate the functional correlation coefficients, and in particular to show that the new ones are indeed distinct from those already in the literature, we have devised the following example.

\begin{example}\label{example-00}
Let a $3\times 3$ confusion matrix be defined as follows:
\begin{equation}\label{matrix-000}
CM_0=
\bordermatrix{%
      & g_1     & g_2     & g_3     \cr
f_1   & 0.1 & 0 & 0.1 \cr
f_2   & 0.2 & 0 & 0.2 \cr
f_3   & 0 & 0.2 & 0.2 \cr
}.
\end{equation}
With the help of MATLAB's \texttt{fmincon} function (optimization with constraints of nonlinear multi-argument functions), we obtain the following values:
\begin{align*}
\varrho_{\textsc{ii}}(CM_0)   = 0.5345
 \quad \text{when} \quad &
\left\{
    \begin{array}{ll}
\mathbf{f}_{\textsc{ii}} & = (-0.8179,~  -0.8157,~   1.2247)\\
\mathbf{g}_{\textsc{ii}} & = (-1.5275,~   0.6546,~   0.6546)
    \end{array}
  \right.
\\
\varrho_{\textsc{id}}(CM_0) = 0.0000
 \quad \text{when} \quad &
\left\{
    \begin{array}{ll}
\mathbf{f}_{\textsc{id}}& =   (-1.7931,~  -0.0469,~   0.9434)\\
\mathbf{g}_{\textsc{id}}& =   (1.0000,~   0.9999,~  - 0.9999)
    \end{array}
  \right.
\\
\varrho_{\textsc{co}}(CM_0) = 0.5345
 \quad \text{when} \quad &
\left\{
    \begin{array}{ll}
\mathbf{f}_{\textsc{co}} & = (-0.8179,~  -0.8157,~   1.2247)\\
\mathbf{g}_{\textsc{co}} & = (-1.5275,~   0.6546,~   0.6546)
    \end{array}
  \right.
\\
\varrho_{\textsc{anti}}(CM_0) = 0.6123
 \quad \text{when} \quad &
\left\{
    \begin{array}{ll}
\mathbf{f}_{\textsc{anti}}& = (0.8149,~   0.8172,~  -1.2247)\\
\mathbf{g}_{\textsc{anti}}& = (0.4999,~  -1.9999,~   0.5000)
    \end{array}
  \right.
\\
\varrho_{\textsc{sup}}(CM_0) = 0.7071
 \quad \text{when} \quad &
\left\{
    \begin{array}{ll}
\mathbf{f}_{\textsc{sup}}& = (-0.8164,~  -0.8164,~   1.2247)\\
\mathbf{g}_{\textsc{sup}}& = (-1.1547,~   1.7320,~  -0.0000)
    \end{array}
  \right.
\end{align*}
In particular, from these results and equations~(\ref{st-2aa}) and~(\ref{st-2bb}), we have
\begin{equation}\label{eq-00}
\varrho_{\textsc{mon}}(CM_0)= 0.5345
\quad \textrm{and} \quad
\varrho_{\textsc{coanti}}(CM_0)= 0.6123 .
\end{equation}
Hence, the COANTI-correlation coefficient can indeed differ from the MON- and SUP-correlation coefficients; recall bounds~\eqref{coanti-0} and the discussion around them. As we have already noted above, all our calculations are at the six decimal-digit precision, even though we are reporting only the first four decimal digits. In Figure~\ref{fig-quartet}, we have depicted $\mathbf{f}$ and $\mathbf{g}$ corresponding to the ID-, CO-, ANTI- and SUP-correlation coefficients using the straight lines connecting the points $(1,f_1), (2,f_2),(3,f_3)$ as well as those connecting the points $(1,g_1), (2,g_2),(3,g_3)$.
\begin{figure}[h!]
  \centering
  \subfigure[$\mathbf{f}_{\textsc{id}}$ (solid) and $\mathbf{g}_{\textsc{id}}$ (dashed)]{%
    \includegraphics[height=0.26\textwidth]{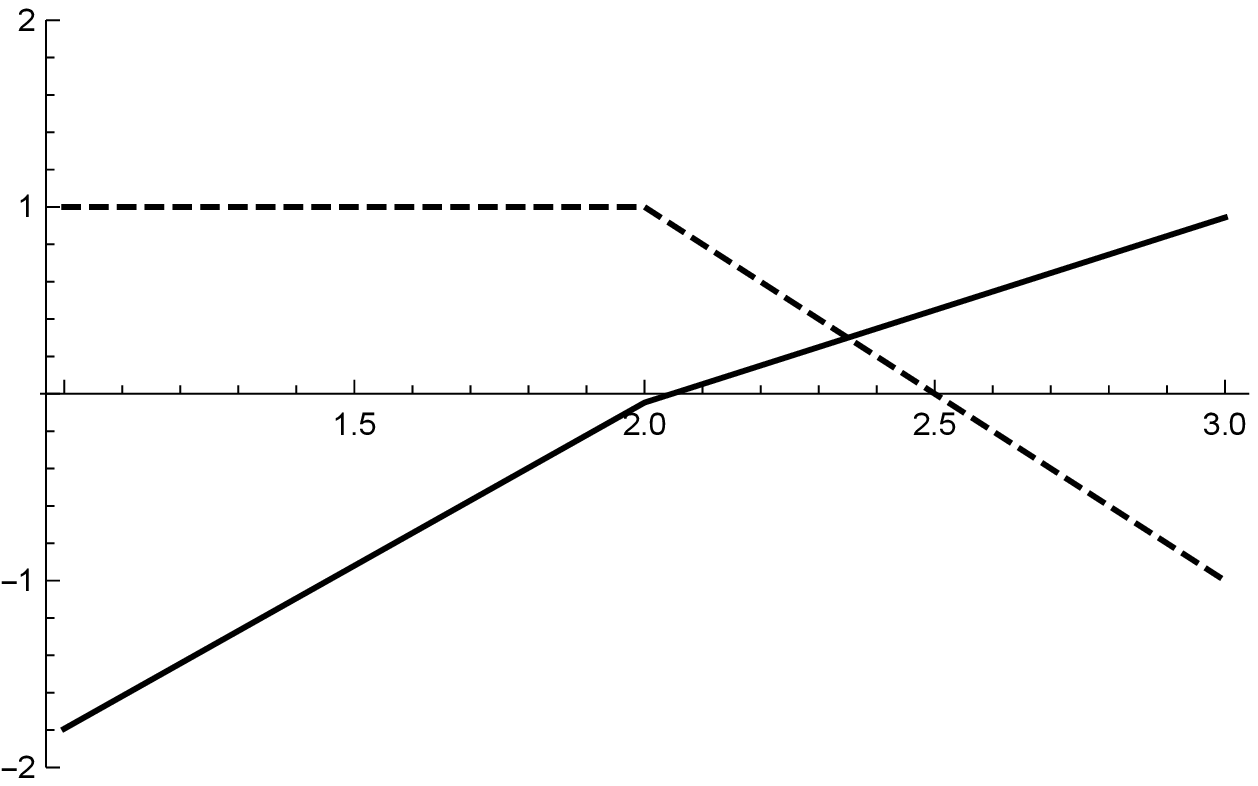}}%
\hspace{0.5in}
  \subfigure[$\mathbf{f}_{\textsc{co}}$ (solid) and $\mathbf{g}_{\textsc{co}}$ (dashed)]{%
    \includegraphics[height=0.26\textwidth]{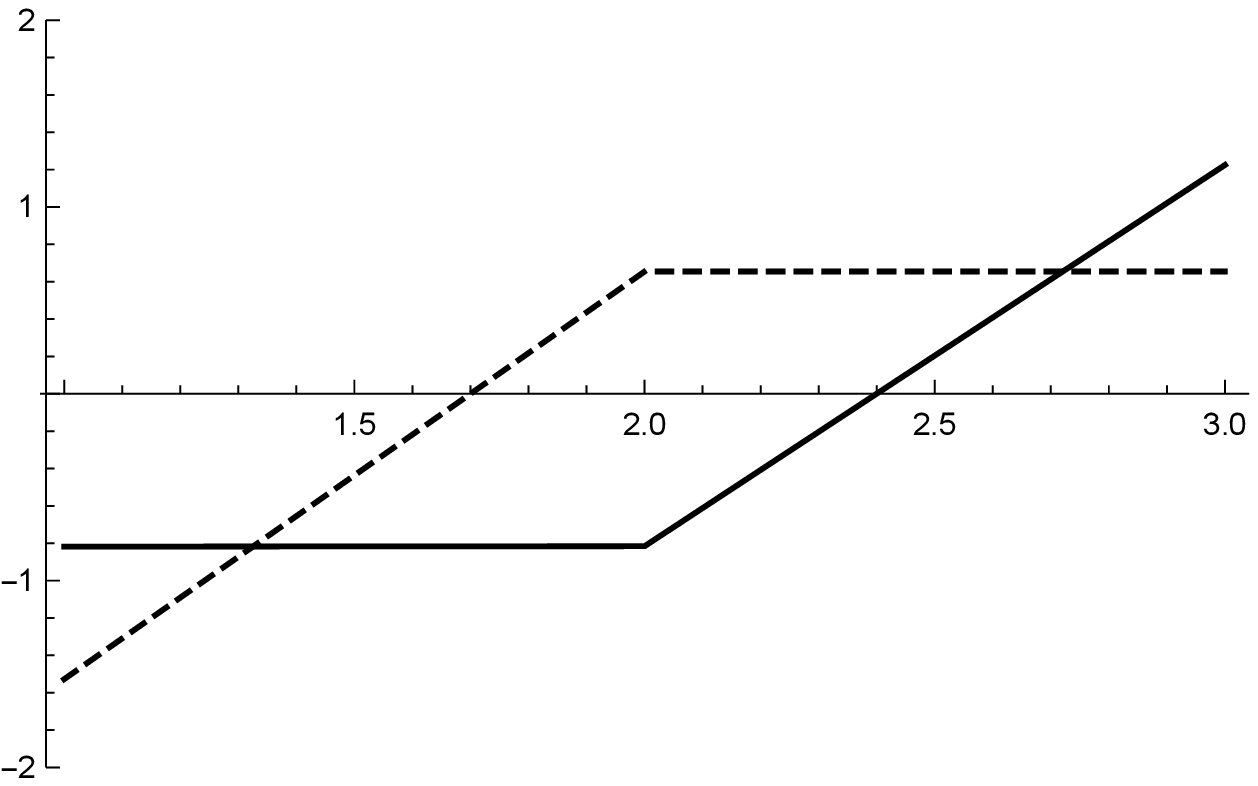}}%
\hfill
  \subfigure[$\mathbf{f}_{\textsc{anti}}$ (solid) and $\mathbf{g}_{\textsc{anti}}$ (dashed)]{%
    \includegraphics[height=0.26\textwidth]{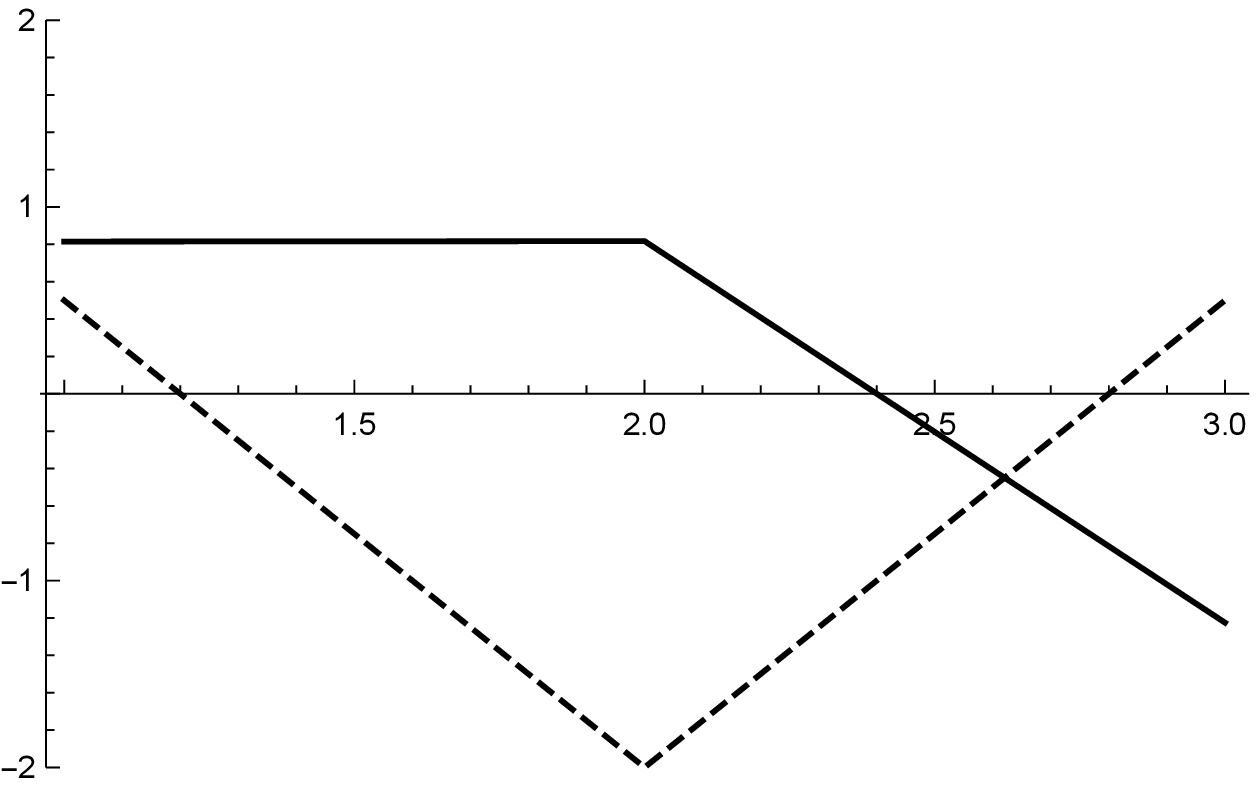}}%
\hspace{0.5in}
  \subfigure[$\mathbf{f}_{\textsc{sup}}$ (solid) and $\mathbf{g}_{\textsc{sup}}$ (dashed)]{%
    \includegraphics[height=0.26\textwidth]{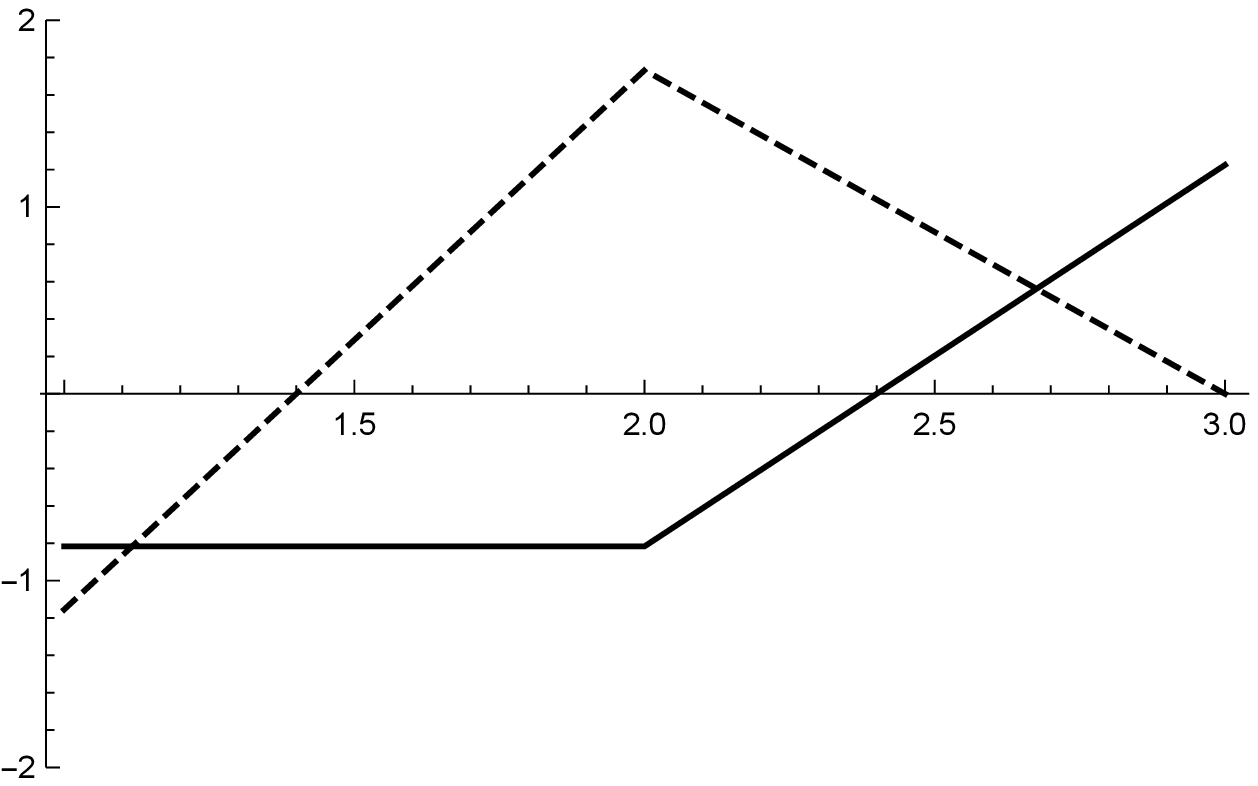}}%
\vspace*{8pt}
\caption{Vectors $\mathbf{f}$ and $\mathbf{g}$ visualized as lines that realize the ID-, CO-, ANTI- and SUP-correlation coefficients in the case of confusion matrix~(\ref{matrix-000}).}
\label{fig-quartet}
\end{figure}
\end{example}

\section{Functional correlation coefficients in action}
\label{sect-5}

As noted above, earlier works by statisticians and social scientists produced fairly satisfactory classification results based on the MON-correlation coefficient, but examples such as those by \citet{T18} initiate a search for more refined tools.

The classical functional correlations (recall Section~\ref{sect-2}), apparently forgotten for quite some time, were rejuvenated by \citet{CP07} and \citet{CS11}. As natural consequences of these studies, in this paper we have suggested the CO-, ANTI-, and COANTI-correlation coefficients, which alongside the classical coefficients, give rise to an efficient and accurate assessment of confusion matrices. To streamline this task, we next propose a flowchart.

\subsection{A flowchart for comparing confusion matrices}
\label{sect-5.1}

We are now in the position to revisit the thirteen confusion matrices of \citet{CS11} by means of the flowchart depicted in Figure~\ref{flowchart}.
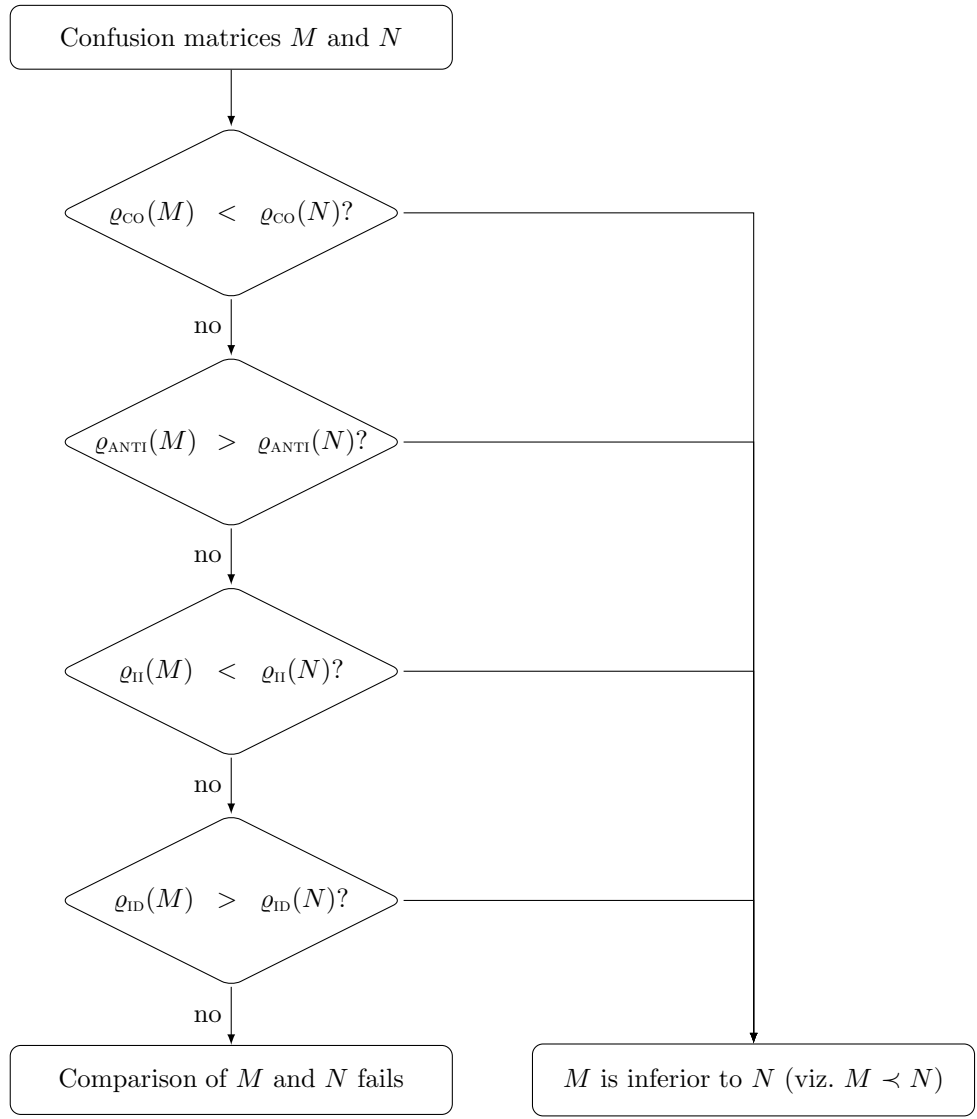
\begin{figure}[h!]
  \centering
\begin{tikzpicture}[-latex]
  \matrix (chart)
    [
      matrix of nodes,
      column sep      = 3em,
      row sep         = 5ex,
      column 1/.style = {nodes={decision}},
      column 2/.style = {nodes={env}}
    ]
    {
      |[root]| Confusion matrices $M$ and $N$          \\
$\varrho_{\textsc{co}}(M)<\varrho_{\textsc{co}}(N)$?    \\
$\varrho_{\textsc{anti}}(M)>\varrho_{\textsc{anti}}(N)$?    \\      $\varrho_{\textsc{ii}}(M)<\varrho_{\textsc{ii}}(N)$?    \\      $\varrho_{\textsc{id}}(M)>\varrho_{\textsc{id}}(N)$?    \\
      |[finish]|  Comparison of $M$ and $N$ fails
      &  |[finish]| $M$ is inferior to $N$ (viz. $M \prec N$)  \\
    };
  \draw(chart-1-1) edge (chart-2-1)
   \foreach \x/\y in {2/3, 3/4, 4/5, 5/6} {(chart-\x-1) \no (chart-\y-1) };
  \foreach \x in {2,...,5} {\draw (chart-\x-1) -| (chart-6-2);}
\end{tikzpicture}
  \caption{A flowchart for comparing confusion matrices.}\label{flowchart}
\end{figure}
A few remarks concerning the flowchart are in order (recall the notes at the end of Section~\ref{sect-3} for enhancing intuition).

First, it is natural to view a classifier superior to, or at least not inferior than, any other if its CO-correlation coefficient is equal to 1, which happens, for example, in the case of the diagonal confusion matrix, as noted below equation~\eqref{st-2bb}. Hence, we view a confusion matrix $M$ inferior to another matrix $N$ when $\varrho_{\textsc{co}}(M)<\varrho_{\textsc{co}}(N)$.

Second, it may happen that $\varrho_{\textsc{co}}(M)=\varrho_{\textsc{co}}(N)$. In this case, it is  natural to seek help from the ANTI-correlation coefficient. (We should point out that antimonotonicity is not an antonym of comonotonicity: there are vectors which are neither comonotone nor antimonotone.) We view $M$ inferior to $N$ whenever $\varrho_{\textsc{anti}}(M)>\varrho_{\textsc{anti}}(N)$, which is reflected in the second step of the flowchart.

We also encounter confusion matrices for which we have $\varrho_{\textsc{co}}(M)=\varrho_{\textsc{co}}(N)$ and $\varrho_{\textsc{anti}}(M)=\varrho_{\textsc{anti}}(N)$. In this case, we seek help from the classical II- and ID-correlation coefficients, as pointed out in the third and fourth steps of the flowchart, which mimic the first and second steps, respectively.

\subsection{The matrices of \citet{CS11} revisited}

We start with the matrices $CM_{1}$--$CM_{6}$ \citep[p.~1185]{CS11}. The first pair that we compare are $CM_{1}$ and $CM_{2}$. We have:
\begin{align*}
CM_1 &=
\bordermatrix{%
      & g_1     & g_2     & g_3   \cr
f_1   & 0.2 &                  0 &  0.1 \cr
f_2   & 0.1 &  0.1 &                  0 \cr
f_3   & 0.2 &  0.1 &  0.2 \cr
}
\Longrightarrow
\left\{
\begin{array}{ll}
     \varrho_{\textsc{ii}}(CM_1)  &=  0.2309\\
     \varrho_{\textsc{id}}(CM_1) &=  0.0476\\
                 \varrho_{\textsc{co}}(CM_1) &= 0.4330\\
               \varrho_{\textsc{anti}}(CM_1) &=  0.0476\\
               \varrho_{\textsc{sup}}(CM_1) &= 0.4537\\
\end{array}
\right.
\\
\\
CM_2 &=
\bordermatrix{%
      & g_1     & g_2     & g_3   \cr
f_1   & 0.1  &                 0    &               0 \cr
f_2   & 0                  &    0.4 &               0 \cr
f_3   & 0.2  & 0.2    & 0.1 \cr
}
\Longrightarrow
\left\{
\begin{array}{ll}
     \varrho_{\textsc{ii}}(CM_2) &= 0.5091\\
     \varrho_{\textsc{id}}(CM_2) &= 0.2182\\
                 \varrho_{\textsc{co}}(CM_2) &= 0.7165\\
               \varrho_{\textsc{anti}}(CM_2) &=  0.2182\\
               \varrho_{\textsc{sup}}(CM_2) &=  0.7165\\
\end{array}
\right.
\end{align*}
Since $\varrho_{\textsc{co}}(CM_{1})<\varrho_{\textsc{co}}(CM_{2})$, we conclude that $CM_{1} \prec CM_{2}$, which coincides with the conclusion by \citet{CS11}.

Following the discussion of \citet{CS11}, we next compare $CM_{3}$ and $CM_{4}$, for which we have the following results:
\begin{align*}
CM_3 &=
\bordermatrix{%
      & g_1    & g_2     & g_3    \cr
f_1   & 0.1428 &  0       &  0.1428 \cr
f_2   & 0      &  0       &  0 \cr
f_3   & 0.3285 &  0.2857 &  0 \cr
}
\Longrightarrow
\left\{
\begin{array}{ll}
     \varrho_{\textsc{ii}}(CM_3) &=  -0.0912\\
     \varrho_{\textsc{id}}(CM_3) &=  0.6454\\
                 \varrho_{\textsc{co}}(CM_3) &= 0.3999\\
               \varrho_{\textsc{anti}}(CM_3) &= 0.6892\\
               \varrho_{\textsc{sup}}(CM_3) &=  0.6892\\
\end{array}
\right.
\\
\\
CM_4 &=
\bordermatrix{%
      & g_1     & g_2     & g_3    \cr
f_1   & 0.1428   &    0                 &  0.1428 \cr
f_2   & 0                   &    0.2857 &  0.1428 \cr
f_3   & 0.1428   &    0.1428 &                  0 \cr
}
\Longrightarrow
\left\{
\begin{array}{ll}
     \varrho_{\textsc{ii}}(CM_4) &=  0.2999\\
     \varrho_{\textsc{id}}(CM_4) &=  0.3281\\
                 \varrho_{\textsc{co}}(CM_4) &=  0.5902\\
               \varrho_{\textsc{anti}}(CM_4) &= 0.3281\\
               \varrho_{\textsc{sup}}(CM_4) &= 0.5902\\
\end{array}
\right.
\end{align*}
Since $\varrho_{\textsc{co}}(CM_{3})<\varrho_{\textsc{co}}(CM_{4})$, we have $CM_{3} \prec CM_{4}$, which coincides with the conclusion by \citet{CS11}.

We now compare $CM_{5}$ and $CM_{6}$ of \citet{CS11}:
\begin{align*}
CM_5 &=
\bordermatrix{%
      & g_1     & g_2     & g_3   & g_4     \cr
f_1   & 0.1428 &  0.1428 &  0       &        0 \cr
f_2   & 0      &  0.1428 &  0       &        0.1428 \cr
f_3   & 0      &  0      &  0       &        0.3285 \cr
f_4   & 0      &  0      &  0       &        0 \cr
}
\Longrightarrow
\left\{
\begin{array}{ll}
     \varrho_{\textsc{ii}}(CM_5) &= 0.8660\\
     \varrho_{\textsc{id}}(CM_5) &= -0.3535\\
                 \varrho_{\textsc{co}}(CM_5) &=  0.8660\\
               \varrho_{\textsc{anti}}(CM_5) &= 0.8416\\
               \varrho_{\textsc{sup}}(CM_5) &= 0.8660\\
\end{array}
\right.
\\
\\
CM_6 &=
\bordermatrix{%
      & g_1     & g_2     & g_3   & g_4     \cr
f_1   & 0                 &  0                    &    0.1428 &   0 \cr
f_2   & 0.1428 &  0.1428    &    0.1428 &   0 \cr
f_3   & 0.1428 &  0.1428    &    0.1428 &   0 \cr
f_4   & 0                 &  0                    &    0                 &   0 \cr
}
\Longrightarrow
\left\{
\begin{array}{ll}
     \varrho_{\textsc{ii}}(CM_6) &= -0.0912\\
     \varrho_{\textsc{id}}(CM_6) &=  0.4714\\
                 \varrho_{\textsc{co}}(CM_6) &= 0.2581\\
               \varrho_{\textsc{anti}}(CM_6) &= 0.4714\\
               \varrho_{\textsc{sup}}(CM_6) &= 0.4714\\
\end{array}
\right.
\end{align*}
Since $\varrho_{\textsc{co}}(CM_{5})>\varrho_{\textsc{co}}(CM_{6})$, we have $CM_{5} \succ CM_{6}$, which coincides with the conclusion by \citet{CS11}.

We next explore $CM_{10}$, $CM_{11}$ and $CM_{12}$ of \citet[p.~1187]{CS11}:
\begin{align*}
CM_{10} &=
\bordermatrix{%
      & g_1     & g_2     & g_3   & g_4   & g_5   \cr
f_1   & 0 &                  0 &                  0 &                  0       &        0 \cr
f_2   & 0 &  0.2083 &  0.0291 &                  0       &        0 \cr
f_3   & 0 &  0.0083 &  0.3916 &  0.0083       &        0 \cr
f_4   & 0 &                  0 &  0.0458 &  0.1625       &        0 \cr
f_5   & 0 &                  0 &                  0 &  0.0208       &        0.1250 \cr
}
\Longrightarrow
\left\{
\begin{array}{ll}
     \varrho_{\textsc{ii}}(CM_{10}) &= 0.9459\\
     \varrho_{\textsc{id}}(CM_{10}) &=  -0.2109\\
                 \varrho_{\textsc{co}}(CM_{10}) &=0.9459\\
               \varrho_{\textsc{anti}}(CM_{10}) &=  -0.0512\\
               \varrho_{\textsc{sup}}(CM_{10}) &=0.9459\\
\end{array}
\right.
\\
\\
CM_{11} &=
\bordermatrix{%
      & g_1     & g_2     & g_3   & g_4  & g_5    \cr
f_1   & 0    &               0  &                 0 &                  0 &                  0 \cr
f_2   & 0    &               0  & 0.1875 &  0.0500 &                  0 \cr
f_3   & 0    &               0  & 0.0083 &  0.3625 &  0.0375 \cr
f_4   & 0    &               0  &                 0 &  0.0250 &  0.1833 \cr
f_5   & 0    &               0  &                 0 &                  0 &  0.1458 \cr
}
\Longrightarrow
\left\{
\begin{array}{ll}
     \varrho_{\textsc{ii}}(CM_{11}) &=  0.8966\\
     \varrho_{\textsc{id}}(CM_{11}) &= -0.2039\\
                 \varrho_{\textsc{co}}(CM_{11}) &= 0.8966\\
               \varrho_{\textsc{anti}}(CM_{11}) &= 0.8434\\
               \varrho_{\textsc{sup}}(CM_{11}) &= 0.8966\\
\end{array}
\right.
\\
\\
CM_{12} &=
\bordermatrix{%
      & g_1     & g_2     & g_3   & g_4  & g_5    \cr
f_1   & 0 &                  0 &                  0 &                  0 &                  0 \cr
f_2   & 0 &  0.2083 &  0.0291 &                  0 &                  0 \cr
f_3   & 0 &  0.0083 &  0.3916 &  0.0083 &                  0 \cr
f_4   & 0 &                  0 &  0.0875 &  0.1208 &                  0 \cr
f_5   & 0 &                  0 &                  0 &  0.1208 &  0.0250 \cr
}
\Longrightarrow
\left\{
\begin{array}{ll}
     \varrho_{\textsc{ii}}(CM_{12}) &=  0.9096\\
     \varrho_{\textsc{id}}(CM_{12}) &= -0.2173\\
                 \varrho_{\textsc{co}}(CM_{12}) &=  0.9096\\
               \varrho_{\textsc{anti}}(CM_{12}) &= 0.5520\\
               \varrho_{\textsc{sup}}(CM_{12}) &= 0.9096\\
\end{array}
\right.
\end{align*}
According to the flowchart of Figure~\ref{flowchart}, we first compare the CO-correlation coefficients of the three matrices and conclude that $CM_{10} \succ CM_{11}$, $CM_{11} \prec CM_{12}$, and $CM_{10} \succ CM_{12}$. Hence, $CM_{10} $ exhibits the best result, which coincides with the conclusion of \citet{CS11}. Furthermore, $CM_{11} $ exhibits the worst result among the three matrices, which also coincides with the conclusion of \citet{CS11}.

Finally, we look at the confusion matrices $CM(A)$--$CM(D)$ of \citet[p.~1177]{CS11}. The matrices are the most problematic ones for the classifiers of the aforementioned paper, and they are also troublesome for the flowchart of Figure~\ref{flowchart}. We begin with results:
\begin{align*}
CM(A) &=
\bordermatrix{%
      & g_1     & g_2     & g_3   & g_4     \cr
f_1   & 0.3076         & 0                  & 0  & 0                 \cr
f_2   & 0                         & 0.4615  & 0  & 0                 \cr
f_3   & 0                         & 0                  & 0  & 0                 \cr
f_4   & 0                         & 0                  & 0  & 0.2307 \cr
}
\Longrightarrow
\left\{
\begin{array}{ll}
    \varrho_{\textsc{ii}}(CM(A)) &= 1\\
     \varrho_{\textsc{id}}(CM(A)) &=-0.3651\\
                 \varrho_{\textsc{co}}(CM(A)) &=1\\
               \varrho_{\textsc{anti}}(CM(A)) &= -0.3651\\
               \varrho_{\textsc{sup}}(CM(A)) &=  1\\
\end{array}
\right.
\\
\\
CM(B) &=
\bordermatrix{%
      & g_1     & g_2     & g_3   & g_4     \cr
f_1   &  0  & 0.3076 &                  0 &                  0 \cr
f_2   &  0  &                 0 &  0.4615 &                  0 \cr
f_3   &  0  &                 0 &                  0 &                  0 \cr
f_4   &  0  &                 0 &                  0 &  0.2307 \cr
}
\Longrightarrow
\left\{
\begin{array}{ll}
     \varrho_{\textsc{ii}}(CM(B)) &= 1\\
     \varrho_{\textsc{id}}(CM(B)) &= -0.3651\\
                 \varrho_{\textsc{co}}(CM(B)) &=1\\
               \varrho_{\textsc{anti}}(CM(B)) &= 1\\
               \varrho_{\textsc{sup}}(CM(B)) &= 1\\
\end{array}
\right.
\\
\\
CM(C) &=
\bordermatrix{%
      & g_1     & g_2     & g_3   & g_4     \cr
f_1   & 0  &                 0 &  0.3076 &                  0 \cr
f_2   & 0  &                 0 &  0.4615 &                  0 \cr
f_3   & 0  &                 0 &                  0 &                  0 \cr
f_4   & 0  &                 0 &                  0 &  0.2307 \cr
}
\Longrightarrow
\left\{
\begin{array}{ll}
     \varrho_{\textsc{ii}}(CM(C)) &=  1\\
     \varrho_{\textsc{id}}(CM(C)) &=  -0.3651\\
                 \varrho_{\textsc{co}}(CM(C)) &=  1\\
               \varrho_{\textsc{anti}}(CM(C)) &=  1\\
               \varrho_{\textsc{sup}}(CM(C)) &= 1\\
\end{array}
\right.
\\
\\
CM(D) &=
\bordermatrix{%
      & g_1     & g_2     & g_3   & g_4     \cr
f_1   & 0                 &  0.3076  & 0 &                  0 \cr
f_2   & 0.4615 &  0                  & 0 &                  0 \cr
f_3   & 0                 &  0                  & 0 &                  0 \cr
f_4   & 0                 &  0                  & 0 &  0.2307 \cr
}
\Longrightarrow
\left\{
\begin{array}{ll}
\varrho_{\textsc{ii}}(CM(D))   &= 1 \\
\varrho_{\textsc{id}}(CM(D))   &= 0.6172 \\
\varrho_{\textsc{co}}(CM(D))   &= 1 \\
\varrho_{\textsc{anti}}(CM(D)) &= 1  \\
\varrho_{\textsc{sup}}(CM(D))  &= 1\\
\end{array}
\right.
\end{align*}
For every $M\in \{CM(A),CM(B),CM(C),CM(D)\}$, we have $\varrho_{\textsc{co}}(M)=1$  and thus the first step of the flowchart fails to differentiate between the matrices. The second step singles out the matrix $CM(A)$ as superior because $\varrho_{\textsc{anti}}(CM(A))<\varrho_{\textsc{anti}}(M)$ for all $M\in \{CM(B),CM(C),CM(D)\}$. To compare the performance of the latter three matrices, we first observe that the third step of the flowchart fails to introduce clarity because $\varrho_{\textsc{ii}}(M)=1$ for all the matrices. Hence, we need to move on to the fourth step, which suggests that $CM(D)$ is inferior to both $ CM(B)$ and $CM(C)$ because $\varrho_{\textsc{id}}(CM(D))>\varrho_{\textsc{id}}(M)$ for $M\in \{CM(B),CM(C)\}$. The flowchart does not, however, succeed in differentiating between the matrices $ CM(B)$ and $CM(C)$ because their respective correlation coefficients coincide. We conclude the discussion concerning the four matrices $CM(A)-CM(D)$ by noting that all the comparisons we have achieved with the help of the functional correlation coefficients and the flowchart of Figure~\ref{flowchart} are in line with the findings of \citet[pp.~1177--1178]{CS11}.

\subsection{An alternative computational algorithm}

All numerical results of this paper have been double-checked using an alternative numerical procedure, which we  describe next. The starting point for the procedure is definition (\ref{corr-fg}) of $C(\mathbf{f},\mathbf{g})$. We simulate two independent $d$-dimensional Gaussian random vectors, both with independent and standard-Gaussian coordinates.

The choice of the $d$-dimensional Gaussian distribution is natural because after the Euclidean normalization of both $\mathbf{f}$ and $\mathbf{g}$, which can always be done without loss of generality in the context of calculating $C(\mathbf{f},\mathbf{g})$, each of the normalized vectors $\mathbf{f}/\Vert \mathbf{f} \Vert $ and $\mathbf{g}/\Vert \mathbf{g} \Vert $ uniformly fill in the unit sphere $S^{d-1}$ \citep[e.g.,][]{M72}.

These Gaussian vectors play the roles of $\mathbf{f}$ and $\mathbf{g}$: if they belong to an appropriate set $\mathcal{S}$ (for example, $\mathcal{S}_{\textsc{ii}} $ when calculating $\varrho_{\textsc{ii}}(M_{X,Y})$), then we calculate $C(\mathbf{f},\mathbf{g})$ and record its value. We repeat the procedure as many times as needed to get $10^6$ values of $C(\mathbf{f},\mathbf{g})$. Finally, we find the maximal value among the obtained $10^6$ values and declare it an approximate value of the functional correlation coefficient under consideration (for example, $\varrho_{\textsc{ii}}(M_{X,Y})$).

\section{Concluding notes}
\label{sect-conclude}

In this paper we have recalled several classical \citep[e.g.,][and references therein]{S92} and also introduced and justified new functional correlation coefficients (CO, ANTI, and COANTI) to aid the assessment of confusion matrices. We have explored properties of these coefficients and illustrated their performance using a number of confusion matrices \citep{CS11}. Our suggested classification algorithm, in the form of a flowchart, is based entirely on these functional correlation coefficients, and it has reached the same conclusions as those by \citet{CS11}, who have used a number of diverse (dis)similarity indices.

Finally, we note that the proposed flowchart in Figure \ref{flowchart}, with its specific choices and arrangement of functional correlation coefficients, should not be viewed as definitive. Indeed, even though we have used, for example, the CO-correlation coefficient at the very top of the flowchart, due to our subject-matter knowledge gained from the research by, among others, \citet{CP07}, \citet{CS11}, and \citet{T18}, there might be contexts, such as classification problems related to social mobility \citep[e.g.,][]{BH06,CJL09}, where placing the II- or some other functional correlation coefficient at the very top of the flowchart would be more appropriate.

\section*{Acknowledgements}

We are grateful to Managing Editor Gernot A. Fink, an Associate Editor, and all the reviewers for the many comments, suggestions, and constructive criticism that guided our work on the revision, which has been substantial. Francesca Greselin's interest in, and enthusiasm about, classification and related problems have been contagious. Conversations with Ruodu Wang on various facets and applications of comonotonicity have been enlightening. We are indebted to Hyukjun (Jay) Gweon for the many fruitful discussions that we have had on the topic of this paper, as well as for his insights and suggestions related to our results. 

Research of the second author has been supported by the Natural Sciences and Engineering Research Council of Canada under the title ``From Data to Integrated Risk Management and Smart Living: Mathematical Modelling, Statistical Inference, and Decision Making,'' as well as by a Mitacs Accelerate Award from the national research organization Mathematics of Information Technology and Complex Systems, Canada, in partnership with Unilever Canada Inc., under the title ``Price Elasticities and Promotion Cannibalization Effect on Promotion Activities.''

%\vspace*{-0.01in}
%%\vspace*{-0.3in}
%\noindent
%\rule{12.6cm}{.1mm}

%\section*{Biographical Sketch and Photo}

%Upon acceptance of an article, a brief biographical sketch and
%photograph of each author are to be supplied to the Publisher.
%
%\biophoto{wang}{{\bf Chuan-Cheng Wang} received the
%B.S.~degree\break
%in electrical engineering from National Sun Yat-Sen University,
%Kaoh-\break siung, Taiwan in 1992 and the M.S. degree in
%computer science from National Chiao Tung University, Hsinchu,\break
%Taiwan in 2001.}
%
%\vglue-1.75truein
%\hspace*{2.45truein}
%\biophoto{gao}{{\bf Yongsheng Gao} received the B.Sc. and\break
%M.Sc. degrees in electronic engineering from\break Zhejiang
%University,\break China, in 1985 and 1988 respectively, and the
%Ph.D. in computer engineering from Nanyang Technological University,
%Singapore. Currently, he is an assistant professor with Nanyang
%Technological University, Singapore.}

\end{document}